\documentclass[letterpaper, 10pt]{article} \usepackage{amsthm} \usepackage{setspace} \newcommand{\numberofauthors}[1]{} \newcommand{\alignauthor}{\and} \newcommand{\affaddr}[1]{#1} \newcommand{\email}[1]{\href{mailto:#1}{\texttt{#1}}} \newcommand{\balancecolumns}{} \newcommand{\proofword}{Proof} \newcommand{\stdvers}[1]{#1} \newcommand{\jrnvers}[1]{} \usepackage[margin=1.25 in]{geometry} \usepackage{hyperref} 
\usepackage{amsfonts}
\usepackage{amsmath}

\newcommand{\abs}[1]{\ensuremath{\left| #1 \right|}}
\newcommand{\cov}{\ensuremath{\mathrm{Cov}}}
\newcommand{\cum}{\ensuremath{\mathrm{Cum}}}
\newcommand{\diag}{\mathrm{diag}}
\newcommand{\dist}{\mathrm{dist}}
\newcommand{\kmax}{\kappa_{\max}}
\newcommand{\kmin}{\kappa_{\min}}

\newcommand{\nCr}[2]{{{#1}\choose{#2}}}
\newcommand{\norm}[1]{\ensuremath{\lVert #1 \rVert}}
\newcommand{\norms}[1]{{\lVert#1\rVert}^2}

\newcommand{\var}{\ensuremath{\mathrm{Var}}}
\renewcommand{\vec}[1]{\ensuremath{\mathbf{#1}}}
\newcommand{\gvec}[1]{\ensuremath{\boldsymbol{#1}}}

\newcommand{\E}{\ensuremath{\mathbb{E}}}

\newcommand{\R}{\ensuremath{\mathbb{R}}}

\newcommand{\expectation}{\operatorname{\mathbb{E}}}
\newcommand{\e}{\expectation}
\newcommand{\eps}{\epsilon}

\newcommand{\deletion}[2]{}
\newcommand{\insertion}[2]{#2}
\newcommand{\replacement}[3]{#3}

\newtheorem{thm}{Theorem}[section]

\newtheorem{lemma}[thm]{Lemma}

\newtheorem{defn}[thm]{Definition}

\numberofauthors{3}
\author{
Mikhail Belkin \\
        \affaddr{Ohio State University}  \\
        \affaddr{Computer Science and Engineering,} \\
        \affaddr{2015 Neil Avenue, Dreese Labs 597.}  \\
        \affaddr{Columbus, OH 43210} \\
        \email{mbelkin@cse.ohio-state.edu}
\alignauthor
Luis Rademacher \\
        \affaddr{Ohio State University}  \\
        \affaddr{Computer Science and Engineering,} \\
        \affaddr{2015 Neil Avenue, Dreese Labs 495.}  \\
        \affaddr{Columbus, OH 43210} \\
        \email{lrademac@cse.ohio-state.edu}
\and
\alignauthor
James Voss \\
        \affaddr{Ohio State University}  \\
        \affaddr{Computer Science and Engineering,} \\
        \affaddr{2015 Neil Avenue, Dreese Labs 586.}  \\
        \affaddr{Columbus, OH 43210} \\
        \email{vossj@cse.ohio-state.edu}
}
\title{ Blind Signal Separation in the Presence of Gaussian Noise}

\begin{document}
\maketitle
\begin{abstract}
\smallskip
A prototypical blind signal separation problem  is the so-called {\it cocktail party problem}, with $n$ people talking simultaneously and $n$ different microphones within a room. The goal is to  recover  each speech signal  from the microphone inputs. Mathematically this can be modeled by assuming that we are given samples from a\insertion{JV}{n} $n$-dimensional random variable $\vec X=A \vec S$, where $\vec S$ is a vector whose coordinates are independent random variables corresponding to each speaker. The objective is to recover the  matrix $A^{-1}$ given random samples from $\vec X$. A range of techniques  collectively known as Independent Component Analysis (ICA) have been proposed to address this problem in the  signal processing and machine learning literature. Many of these techniques are based on using the kurtosis or other cumulants to recover the components.

In this paper we propose a new algorithm for solving the blind signal separation problem in the presence of additive Gaussian noise, when we are given samples from $\vec X=A\vec S +\gvec \eta$, where $\gvec \eta$ is drawn from an unknown, not necessarily spherical $n$-dimensional Gaussian distribution. Our approach is based on a method for decorrelating a sample with additive Gaussian noise under the assumption that the underlying distribution is a linear transformation of a distribution with independent components. Our decorrelation routine is based on the properties of cumulant tensors and can be combined with any standard cumulant-based method for ICA to get an algorithm that is provably robust in the presence of Gaussian noise. We derive polynomial bounds for \insertion{JV}{the} sample complexity and error propagation of our method.

\end{abstract}
\section{Introduction and related work ~~~~~~~~~~~~~}

A prototypical blind signal separation setting is the so-called \emph{cocktail party problem}: in a room, there are $n$ people speaking simultaneously and $n$ microphones, with each microphone capturing a superposition of the voices.  
The objective is to recover the voice of each individual speaker. The simplest modeling assumption is to consider each speaker as producing a signal to be a random variable independent of the others and to take the superposition  to be a linear transformation independent of time.
This leads to the following problem: given a sample from $n$-dimensional random variable $\vec{X}$, satisfying $\vec X = A \vec S$, where $A$ is a non-singular square matrix and $\vec S$ is another random vector whose coordinates   are unknown independently distributed (but not necessarily identical) random variables, we need to recover the  matrix $A^{-1}$. 
Equivalently, we need to recover the basis corresponding to the directions of the independent components. 

The name Independent Component Analysis refers to a broad range of algorithms  addressing this signal separation problem as well as its variants and extensions. It has generated significant interest and an extensive literature in the signal processing and machine learning communities due to its applicability to a variety of important practical situations including speech~\cite{BSS2007}, vision~\cite{Bell1997} and various biological and medical applications, e.g.,~\cite{Jung2000}.  For a comprehensive introduction see the books~\cite{Comon2010, Hyvarinen2000}. 

One widely used class of algorithms for ICA is based on the remarkable fact that if the data is whitened, that is, $\vec X$ has the zero mean and the identity covariance matrix, then the absolute value of kurtosis reaches its maximum in the  directions corresponding to the independent components. More precisely, consider the kurtosis as a function on the $n$-dimensional unit sphere. For whitened data it can be defined as follows:
\[ 
\vec v \mapsto \kappa_4(\vec v \cdot \vec X) := \e(( \vec v\cdot \vec X)^4) - 3 
\] 
 It can be shown~\cite{Delfosse1995, Frieze1996} that the vectors corresponding to the maxima of the absolute value of $ \kappa_4(\vec v \cdot \vec X) $ form an orthonormal basis whose elements are independent random variables.   Thus the underlying structure of the signal can be recovered by analyzing the behavior of this function. Moreover, computing the kurtosis involves the expected value of the fourth power  of a random variable, which can be easily  approximated from a finite sample.

This observation leads to the following  procedure for the Independent Component Analysis in the noiseless case:\\
{\bf Step 1.}   ``Whiten'' the original signal, that is, apply a linear transformation that transforms the covariance matrix of the sample to the identity.
This is typically achieved by using the Principal Component Analysis (PCA) to transform the input data  to the basis of its principal directions  by an orthogonal transformation  and rescaling the resulting data appropriately. \\
{\bf Step 2.} After the signal is whitened, various optimization procedures can be used to find the maxima of the absolute value of kurtosis over the unit sphere. The independent components are recovered from the directions of these maxima. 

In their recent paper~\cite{Arora2012} Arora, et al. make an important observation that  for a slight variation of Step 2 to work, it is sufficient  for the the sample to be decorrelated (\emph{quasi-whitened}), that is, to have independent coordinates in some orthogonal basis, rather than fully whitened (having the identity covariance matrix). 

In this paper we consider the problem of signal separation for a noisy signal $\vec X=A\vec S +\gvec \eta$, where $\gvec \eta$ is an unknown, not necessarily spherical $n$-dimensional Gaussian distribution. The main difficulty is in Step 1, since the 
principal directions given by PCA are contaminated by the noise and do not generally decorrelate the underlying signal. Interestingly, as a result of the invariance of the kurtosis under the additive Gaussian noise,  Step 2 of the algorithm is still valid and  the usual methods and analyses still apply with minor caveats. 

The main contribution of our paper is addressing the problem of  decorrelating the underlying signal in the presence of noise. 
We show how to approximate a matrix $B$, such that $B^{-1}A$ is diagonal in the basis of independent coordinates. We provide polynomial bounds for the sample complexity and error analysis as well as an analysis of error propagation compatible with any analysis of Step 2. 

 Our approach can be viewed as a noise-invariant version of PCA for the special case when the underlying probability distribution is a product of independent variables. The method is based on the properties of the fourth cumulant tensor, rather than the usual covariance matrix used in PCA.  To the best of our knowledge, this is the first general algorithm for noisy ICA with sample complexity and running time guarantees. Moreover, unlike  methods such as~\cite{Yeredor2000}, our approach is compatible with any optimization procedure for the Step 2.

\noindent{\bf Related work.}
Over the last twenty years blind signal separation\footnote{Also known as Blind Source Separation.} has become a large and active area of research in signal processing and machine learning community. An important class of methods for ICA is based on the properties of kurtosis 
 and other higher-order cumulants. 

Most of these works concentrate on  algorithms, implementations\insertion{JV}{,} and applications and do not provide a sample or running-time complexity analysis for the algorithms. One   such analysis is provided in Frieze, et al.~\cite{Frieze1996}, where the authors address the question of learning a linear transformation, which is equivalent to the ICA problem, and provide a complexity analysis. In  a slightly different context of cryptoanalysis,~\cite{Nguen2009} analyzes a kurtosis-based method for learning a parallelopiped. In~\cite{Vempala2011} the authors analyze a generalized version of ICA for learning higher-dimensional subspace ``juntas'' in the presence of noise.

The problem of blind signal separation in the presence of noise has been an  active topic of research in the machine learning literature.  In particular we would like to point out the  work of Yeredor \cite{Yeredor2000} which proposes an elegant one-step approach for general ICA with Gaussian noise, based on approximating the Hessian of the second characteristic function, namely $v \mapsto \nabla^2_v \log \e_x (e^{v^T x})$,  at a finite number of generic choices of $v$. 
The  recent work of Hsu and Kakade \cite[Section 3, Theorem 3]{Kakade2012} proposes an approach similar to Yeredor's, using the Hessian of the directional kurtosis  instead of  the second characteristic function and makes interesting connections to learning 
Gaussian mixture distributions in high dimension.
Finally, Arora et al.\ \cite{Arora2012} also use a Hessian-based
technique to provide a complete sample complexity analysis for noisy ICA for
the special case when the underlying signal is a uniform distribution over the
$n$-dimensional binary cube $\{-1, 1\}^n$.
\insertion{JV/MB}{The technique of Arora et al.\ also applies when all independent components have kurtosis of the same sign.  However, their technique cannot be used in the general case since it involves extracting the square root of a matrix that is positive definite only under that condition.}
Our approach, based on the full fourth cumulant tensor, does not face
this difficulty.


We also would like to point out that our approach is closely related to the class of tensor methods for data analysis, see e.g.~\cite{Shashua2005, Morton2010}

\section{Properties of Cumulants}\label{sec:cumulants}

\smallskip


Let $\phi_{\vec{X}}(\vec{t})
= \E[\exp(i\vec{t}^T\vec{X})]$, $\vec{t}\in \R^n$ denote the first charateristic function of a $n$-dimensional vector valued random variable
$\vec{X}$, and let $\psi_{\vec{X}}(\vec{t}) = \log (\phi_{\vec{X}}(\vec{t}))$ denote
the second characteristic function of $\vec{X}$.  Cumulants are defined as the coefficients of the Taylor Expansion of the second characteristic function.
Specifically, using the multi-index notation, we write 
\begin{equation*}
    1 + \sum_{r=1}^\infty \sum_{ i_1, \dotsc, i_n \in [n]^r}
	    \frac{1}{r!}i^r\biggl(\prod_{j=1}^r \vec{t}_{i_j}\biggr)\cum(\vec{X}_{i_1}, \dotsc, \vec{X}_{i_r} ) 
	    = \psi_{\vec{X}}(\vec{t}) .
\end{equation*}
%
%
%
For each cumulant $\cum(\vec{X}_{i_1}, \dotsc, \vec{X}_{i_r})$, $r$ is referred to as the order of the cumulant.
Order $r$  cumulants of a random variable $\vec{X}$ can be collected into a cumulant tensor, called the $r^{th}$ cumulant tensor of $\vec{X}$.  For instance, the fourth order cumulant tensor of $\vec{X}$, denoted by $Q_{\vec{X}}$ in this paper, is defined by $(Q_{\vec{X}})_{ijkl} = \cum(\vec{X}_i, \vec{X}_j, \vec{X}_k, \vec{X}_l)$.  Since any simultaneous draw 
of random variables can be viewed as a draw of a single vector-valued random variable, this definition can be used to construct \replacement{JV}{cross cumulants}{cross-cumulants} between arbitrary random variables.
In the univariate case in which $X$ and $t$ are scalars, the notation $\kappa_r(X)$ is used to denote the $r^{th}$ order cumulant $\cum(X, \dotsc, X)$.

Cumulants are similar in flavor to moments, and indeed all cumulants have polynomial expansions in terms of the moments of the same and \replacement{JV}{smalle}{lesser} order.
For example, the fourth cumulant (kurtosis) of a 0-mean
one-dimensional random variable $X$ can be expanded $\kappa_4(X) = \E[X^4] -
3\E[X^2]^2$.  However, cumulants have nice algebraic properties not
shared by moments, properties on which this work relies heavily.
Let $X_1, \dotsc, X_r$ be real-valued random variables.  Then, \replacement{JV}{cross cumulants}{cross-cumulants} are known to manifest the following properties:
\begin{enumerate}
\item (Multilinearity)  If $c_i\in \R$ is a constant, then
\jrnvers{$\cum(X_1, \allowbreak \dotsc, c_iX_i, \dotsc, X_r) = c_i\cum(X_1, \dotsc, X_i, \dotsc, X_r)$.}
    \stdvers{\[\cum(X_1, \allowbreak \dotsc, c_iX_i, \dotsc, X_r) = c_i\cum(X_1, \dotsc, X_i, \dotsc, X_r).\]}  Also, if $Y_i$ is a random variable, then
  \begin{equation*}
    \begin{split}
        &\cum(X_1, \dotsc, X_i+Y_i, \dotsc, X_r) \\
        &\quad=\cum(X_1, \dotsc, X_i, \dotsc, X_r) \jrnvers{\\
        &\quad\quad} + \cum(X_1, \dotsc, Y_i, \dotsc, X_r) .
    \end{split}
  \end{equation*}
\item (Independence) If 2 variables $X_i$ and $X_j$ ($i < j$) are independent random variables,
  then the cross cumulant
    $\cum(X_1, \dotsc, X_i, \dotsc, X_j, \dotsc, X_n)$
  is zero.  Combined with the multilinearity property, this implies that if the variables $Y_1, \dotsc, Y_n$ are independent of $X_1, \dotsc, X_n$, then
  \begin{equation*}
  \begin{split}
    &\cum(X_1 + Y_1, X_2 + Y_2,  \dotsc, X_n + Y_n) \\
    &\quad = \cum(X_1, X_2, \dotsc, X_n) + \cum(Y_1, Y_2, \dotsc, Y_n) .
  \end{split}
  \end{equation*}
\item (Vanishing Gaussians)  The only non-zero cumulant tensors
  of Gaussian random variables are the 1-tensor mean and the 2-tensor covariance matrix.
\end{enumerate}
Note that in  the univariate case, these properties become:
\begin{enumerate}
\item (Additivity)  If $X$ and $Y$ are independent random variables,
  then $\kappa_r(X+Y) = \kappa_r(X) + \kappa_r(Y)$.
\item (Homogeneity of degree $r$)  If $c$ is a constant, then
  $\kappa_r(cX) = c^r\kappa_r(X)$.
\item (Vanishing Gaussians) The only non-zero cumulants of a
  Gaussian random variable are the mean and the variance (the first
  and second order cumulants).
\end{enumerate}

\section{Problem Statement and Main Result}

\smallskip

Let $\vec{x}^{(1)}, \vec{x}^{(2)}, \dotsc, \vec{x}^{(N)} \in \R^n$ be an
i.i.d.\ $N$-sample of vector-valued random variables.
In independent component analysis (ICA) it is assumed that each $\vec{x}^{(i)}$
is generated from a latent random variable $\vec{s}^{(i)}$ via an
unknown mixing matrix $A$ such that
\begin{equation*}
  \vec{x}^{(i)} = A\vec{s}^{(i)} + \gvec{\eta}^{(i)}
\end{equation*}
where $\gvec{\eta}$ is additive noise.  The latent random variable
$\vec{S}$ is typically assumed to be a vector in $\R^n$; though in
principle, it could be a vector in any space $\R^m$ where $m\leq n$.
The individual coordinates of $\vec{S}$ are assumed to be independent
random variables.  $A$ is taken to be a full rank matrix,
$A~\in~\R^{n\times m}$.  It will be assumed for
simplicity that $m=n$, thus making $A$ invertible.
We will further assume that each random
variable $\vec{S}_i$ has variance 1.  Note that this last assumption serves to remove an ambiguity of the problem, since the columns of $A$ could otherwise be chosen to have any scale.  As a result of these
assumptions, \deletion{JV}{the covariance } $\cov(\vec{S}) $ becomes the identity matrix\deletion{JV}{ $I$}.
\insertion{JV}{For convenience, we also assume that $\vec S$ has 0 mean.}

As discussed in the introduction, most ICA algorithms can be broken down into 2 steps.  In the first
step, the independent components are made orthogonal and rescaled such that $\vec{X} = R \vec{S}$ where $R$ is an orthogonal matrix.  This method of decorrelating the independent components is termed whitening.
In the second step, the columns of $R$ (which correspond to independent components) up to sign and order are found.

In the noisy case the main challenge is presented by Step 1, as Step 2 for kurtosis-based methods is naturally invariant to 
Gaussian noise.
Since additive Gaussian noise affects the covariance matrix $\cov(A\vec{S} + \gvec \eta)$, PCA-based whitening fails to orthogonalize the independent components.  It was observed in \cite{Arora2012} that a variation on step 1 could be used. 
It is enough to make the independent components orthogonal without giving
them the same scale.  Whereas true whitening sets $\vec{X} = R\vec{S}$, 
we replace $R$ with $RD$ such that $R$ is orthogonal and $D$ is a diagonal scaling matrix.  
Thus, following \cite{Arora2012}, \emph{quasi-whitening}\footnote{Hyv\"arinen had a different definition of quasi-whitening in \cite{Hyvarinen1999A}.} can be defined as follows:

\begin{defn} \label{def::Quasi-Whitening} A \textit{quasi-whitening}
  matrix is a matrix $W$ such that $WA = RD$ for some orthogonal
  matrix $R$ and nonsingular diagonal matrix $D$.
\end{defn}

We shall now state our main result.
Let $\vec{e}_1, \dotsc, \vec{e}_n$ be the canonical vectors that form a
  basis for the space spanned by the random vector $\vec{S}$.  Let $\kmin = \min_i(\abs{\kappa_4(\vec{S}_i)})$, $\kmax = \max_i(\abs{\kappa_4(\vec{S}_i)})$, and 
  $\mu_k = \max_i(\E[\vec{S}_i^k])$.  \insertion{JV}{Let $\sigma_{\gvec \eta} = \max_{\norm{\vec u} = 1}\sqrt{\vec u^T\Sigma_{\gvec \eta} \vec u}$ where $\Sigma_{\gvec \eta}$ is the covariance of $\gvec \eta$.}
Let $A_i$ denote the $i$\textsuperscript{th} column of matrix $A$. For clarity of the presentation, we use the following machine model for the running time: a random access machine that allows the following exact arithmetic operations over real numbers in constant time: addition, substraction, multiplication, division and square root.

\begin{thm}\label{thm:mainthm}
Let $\eps > 0$ and $\delta \in (0,1)$. Given 
  \[
      N = O\left(\frac{n^{10} \kappa(A)^{16} }{\epsilon^2\delta } \frac{\kmax^2}{\kmin^4} \biggl(\mu_8 +  \frac{\sigma^8_{\gvec{\eta}}}{\sigma_{\min}(A)^8}\biggr) \right)
  \]
samples of $\vec X = A \vec S + \gvec \eta$ we can compute, in time polynomial in $N$ and $n$, an approximate quasi-whitening matrix $\hat B$ so that with probability at least $1-\delta$ over the sample we have
  \begin{enumerate}
    \item For $i \neq j$, 
  \begin{equation} \label{eq:mainthm:cosine}
    -\epsilon \leq \frac{\langle \hat{B}^{-1}A \vec{e}_i, \hat{B}^{-1}A \vec{e}_j \rangle}{\norm{\hat{B}^{-1}A\vec{e}_i}_2\norm{\hat{B}^{-1}A\vec{e}_j}_2} \leq \epsilon    
  \end{equation}
    \item The length of $\vec{e}_j$ is scaled under the transformation $\hat{B}^{-1}A$ as:
      \begin{equation} \label{eq:mainthm:scaling}
        (1-\epsilon) \norms{A_i}_2 \leq \norms{\hat{B}^{-1}A \vec{e}_j}_2 \leq (1+\epsilon) \norms{A_i}_2
      \end{equation}
    \end{enumerate}
\end{thm}
In simpler words, quasi-whitening approximately orthogonalizes the independent components of $\vec X$ and
scales the independent components based on the lengths of the columns of $A$.

We note that existing cumulant-based methods already employed for step 2 in ICA can be modified in reasonably straightforward ways to work under quasi-whitening.
Several popular ICA algorithms including JADE \cite{Cardoso1993} and the kurtosis-based implementation of FastICA \cite{Hyvarinen1999, Hyvarinen2000} are implemented using cumulants.  Since higher order cumulants ignore Gaussian noise, this allows for the creation of a class of new algorithms that are resistant to additive Gaussian noise. \deletion{JV}{The special case where each $\vec{S}$ is drawn uniformly from $\{-1, 1\}^n$ has been done by Arora et al in \cite{Arora2012}.}

To see the validity of fourth cumulant based algorithms for the second step of ICA in the presence of Gaussian noise, we draw from Observation 2 of Frieze et.\ al.\ in \cite{Frieze1996}. 
An interpretation of the statement and proof is that given $\alpha_1, \alpha_2, \dotsc, \alpha_n \in \R$ such that each $\alpha_i \neq 0$ and the function $G(\vec{v}) = \sum_{i=1}^n \vec{v}_i^4 \alpha_i$ such that $\vec{v}$ is restricted to the unit sphere, we have that when there exists some $\alpha_i > 0$, a complete list of local maxima of $G(\vec{v})$ is given by $\{ \pm \vec e_i : \alpha_i>0\}$ (where $\vec e_i$ is the $i$th canonical vector). Similarly, when there exists some $\alpha_i < 0$, a complete list of local minima of $G(\vec{v})$ is given by $\{ \pm \vec e_i : \alpha_i<0\}$.
Using the properties of cumulants, it follows that given $\vec{v} \in \R^n$, 
\begin{equation}
  \label{eq:fourthCumulantProjPursuitForm}
  \kappa_4(\vec{v} \cdot \vec{S}) = \sum_{i=1}^n \vec{v}_i^4 \kappa_4(\vec{S}_i),
\end{equation}
where $\kappa_4(\vec{S}_i)$ takes on the role of $\alpha_i$.  
As such, any algorithm that maximizes $\abs{\kappa_4(\vec{v}\cdot \vec{S})}$ or alternatively $\kappa_4(\vec{v}\cdot \vec{S})^2$ over the unit sphere will find the canonical vectors.  Of course, one cannot work in the coordinate system of $\vec{S}$, but under the assumption of orthogonality provided, one can instead maximize $\abs{\kappa_4(\vec{u}\cdot \vec{X})}$ where $\vec{u}$ is restricted to the unit sphere since
\begin{equation*}
  \kappa_4(\vec{u} \cdot \vec{X}) = \kappa_4(\vec{u}\cdot (RD\vec{S}))
  = \kappa_4((R^T\vec{u})\cdot (D\vec{S}))
\end{equation*}
using that additive Gaussian noise is ignored by cumulants.  
$D\vec{S}$ is simply a rescaling of $\vec{S}$, and
$\kappa_4(\vec{S}_i)$ can be replaced by $\kappa_4(d_{ii} \vec{S}_i)$
in equation \eqref{eq:fourthCumulantProjPursuitForm}.
Using the change of variable $\vec{v} = R^T\vec{u}$, any locally maximal value for $\vec{u}$ will correspond to a column of $R$, thus recovering a component $\vec{S}_i$
up to scaling and noise.
In \cite{Frieze1996}, Observation 2 summarizes a very similar result in the case of true whitening without additive Gaussian noise using the fourth moment instead of fourth cumulant, 
and a mostly correct efficient algorithm and analysis is provided for the fourth moment based on this observation.

\section{How to Achieve Quasi-Whitening} \label{section:Quasi_Whitening}

\smallskip

%
%

Recall that $Q_{\vec{X}}$ denotes the fourth
cumulant tensor of the observed variable $\vec{X}$, with $ijkl^{th}$
entry:
\begin{equation*}
  (Q_{\vec{X}})_{ijkl} = \cum(\vec{X}_i, \vec{X}_j, \vec{X}_k, \vec{X}_l),
\end{equation*}
and define an operation of tensors on matrices $\mathbb{T} \times
\R^{n\times n} \rightarrow\R^{n\times n}$ by:
\begin{equation*}
  (Q_{\vec{X}}\circ M)_{ij} = \sum_{k, l=1}^n \cum(\vec{X}_i,
  \vec{X}_j, \vec{X}_k, \vec{X}_l) m_{lk}.
\end{equation*}

Before proceeding with the argument leading to the construction of a quasi-whitening matrix,
it is worth making several observations about this operation.  First, the
operation can be viewed as matrix-vector multiplication.  Use
multi-indices $\alpha, \beta$ such that $\alpha$ runs over $(i, j)$ and $\beta$ runs over $(l, k)$,
and note that by symmetry, $(Q_{\vec{X}})_{ijkl} = (Q_{\vec{X}})_{ijlk} = (Q_{\vec{X}})_{\alpha \beta}$.
Under this flattening of the tensor $Q_{\vec{X}}$, the operation becomes
matrix-vector multiplication with $M$ taking on the role of the vector
using $m_{lk} = m_{\alpha}$.

The following Lemma describes how the cumulant tensor transforms under a linear change of variable:
\begin{lemma} \label{lemma:COV_Tranformation}
  Given a random vector-valued variable $\vec{Y} \in \R^n$ and matrices $B, M \in \R^{n \times n}$, then
  $Q_{B\vec{Y}}\circ M = B(Q_{\vec{Y}} \circ(B^T M B) )B^T$.
\end{lemma}
  \begin{proof}
    The proof follows primarily from the multilinearity of cumulants:
    \begin{align*}
\jrnvers{    &  (Q_{B\vec{Y}} \circ M)_{ij} \\ }
\stdvers{	      (Q_{B\vec{Y}} \circ M)_{ij} }
	&= \sum_{k, l=1}^n \cum((B\vec{Y})_i, (B\vec{Y})_j, (B\vec{Y})_k, (B\vec{Y})_l) m_{lk} \\
	&= \sum_{k, l=1}^n \sum_{q,r,s,t = 1}^n
	        \cum(b_{iq}\vec{Y}_q, b_{jr}\vec{Y}_r, b_{ks}\vec{Y}_s, b_{lt}\vec{Y}_t) m_{lk} \\
	&= \sum_{k, l=1}^n \sum_{q,r,s,t=1}^n b_{iq} b_{jr}
	        \cum(\vec{Y}_q,\vec{Y}_r,\vec{Y}_s,\vec{Y}_t)  b_{lt} m_{lk} b_{ks}\\
	&= \sum_{q, r, s, t=1}^n b_{iq} b_{jr} \cum(\vec{Y}_q,\vec{Y}_r,\vec{Y}_s,\vec{Y}_t)(B^T M B)_{ts} \\
	&= \sum_{q, r=1}^n b_{iq} b_{jr} (Q_{\vec{Y}} \circ (B^T M B))_{qr},
    \end{align*}
    which can be equivalently written as $Q_{B\vec{Y}} \circ M = B(Q_{\vec{Y}} \circ (B^T M B))B^T$.
  \end{proof}

The above ideas will be useful both in constructing a quasi-whitening matrix in the noiseless case, as
well as in finding an estimate to a quasi-whitening matrix from data.  What follows is  the
construction of a quasi-whitening matrix when one knows the cumulant tensor exactly.

\begin{lemma}
  \label{lemma:cumulantDiagonalization}
  Let $M$ be an arbitrary matrix.  Then, $Q_{\vec{X}} \circ M =
  ADA^T$ where $D$ is a diagonal matrix with $q^{th}$ entry $d_{qq} =
  \kappa_4(\vec{S}_q) A_q^T M A_q$.
\end{lemma}
  \begin{proof}
    This proof will proceed by simplifying $Q_\vec{X} \circ M$ using
    the properties of cumulants.
    \begin{align*}
\jrnvers{    &  (Q_\vec{X} \circ M)_{ij} \\ }
\stdvers{   (Q_\vec{X} \circ M)_{ij} }
	&= \sum_{k, l=1}^n \cum(\vec{X}_i, \vec{X}_j, \vec{X}_k,
	\vec{X}_l) m_{lk} \\
	&= \sum_{k, l=1}^n \cum\left( \sum_{q=1}^n A_{iq} \vec{S}_q + \gvec{\eta}_i, \sum_{q=1}^n  A_{jq}
	\vec{S}_q + \gvec{\eta}_j, \right.  \notag \\
	& \hspace{ \jrnvers{ 1.8 cm } \stdvers{ 2.2 cm } } \left.
	 \sum_{q=1}^n A_{kq} \vec{S}_q + \gvec{\eta}_k, \sum_{q=1}^n A_{lq}
	\vec{S}_q + \gvec{\eta}_l\right) m_{lk} \\
	&= \sum_{k, l=1}^n \sum_{q=1}^n \cum( A_{iq} \vec{S}_q,  A_{jq}
	\vec{S}_q, A_{kq} \vec{S}_q, A_{lq} \vec{S}_q) m_{lk}\\
	&= \sum_{k, l=1}^n \sum_{q=1}^n A_{iq} A_{jq} A_{kq} A_{lq} \cum( \vec{S}_q, \vec{S}_q, \vec{S}_q, \vec{S}_q) m_{lk}, 
    \end{align*}
    where the last two equalities come from the independence,
    multilinearity, and vanishing Gaussian properties.  Switching
    into univariate cumulant notation and rearranging summations
    yields:
    \begin{align*}
      (Q_\vec{X} \circ M)_{ij} &= \sum_{q=1}^n A_{iq}A_{jq} \kappa_4(\vec{S}_q) \sum_{k, l}^n A_{lq} m_{lk} A_{kq} \\
	&= \sum_{q=1}^n A_{iq}A_{jq} \kappa_4(\vec{S}_q) A_q^T M A_q
    \end{align*}
    which has matrix form:
    \begin{equation*}
      Q_\vec{X} \circ M = ADA^T
    \end{equation*}
    where $D$ is a diagonal matrix with diagonal entries 
    $d_{qq} = \kappa_4(\vec{S}_q) A^T_q M A_q$.
  \end{proof}

\begin{thm}
  \label{thm:TensorQuasi-Whitening}
  Let $M$ be the matrix $(Q_{\vec{X}}\circ I)^{-1}$.  Let $B$ be a
  factorization matrix such that $BB^T = Q_{\vec{X}}\circ M$.
  Then, $B^{-1}$ is a Quasi-Whitening matrix.
\end{thm}
  \begin{proof}
    Applying Lemma \ref{lemma:cumulantDiagonalization} gives
    $Q_{\vec{X}} \circ I = AD'A^T$ with $d'_{qq} = \kappa_4(\vec{S}_q) A_q
    \cdot A_q$.  Note that $M = (A^T)^{-1}D'^{-1} A^{-1}$.  Applying
    Lemma \ref{lemma:cumulantDiagonalization} a second time yields
    $Q_{\vec{X}} \circ M = ADA^T$ where $d_{qq} = \kappa_4(\vec{S}_q)
    A_q^T M A_q$ gives the diagonal elements of $D$.  Manipulating
    $d_{qq}$ yields:
    \begin{align*}
      d_{qq} &= \kappa_4(\vec{S}_q) A_q^T (A^T)^{-1}D'^{-1} A^{-1} A_q \\
            &= \kappa_4(\vec{S}_q) \vec{e}^T_q (D')^{-1} \vec{e_q} \\
            &= \kappa_4(\vec{S}_q) [\kappa_4(\vec{S}_q) A_q\cdot A_q]^{-1} \\
            &= \frac{1}{\norm{A_q}_2^2}
    \end{align*}
    Note that $d_{qq}$ is a positive number for each diagonal entry of
    $D$.  $D^{1/2}$ exists and can be
    uniquely defined by taking the positive square root of all
    diagonal entries.  Letting $B$ be any factorization matrix such
    that $BB^T = Q_{\vec{X}} \circ ((Q_{\vec{X}} \circ I)^{-1}) =
    ADA^T$, then $I = B^{-1}AD^{1/2}(B^{-1}AD^{1/2})^T$ gives that
    $B^{-1}AD^{1/2} = R$ for some orthogonal matrix $R$.  Hence,
    $B^{-1}A = RD^{-1/2}$ gives that $B^{-1}$ is a quasi-whitening
    matrix.
  \end{proof}

\section{Estimation of Cumulants}

\smallskip

So far we have shown that given exact knowledge of the fourth
order cumulant tensor for the random variable $\vec{X} = A\vec{S} + \gvec{\eta}$, it
is possible to find a quasi-whitening matrix $B^{-1}$ such that
$B^{-1}A = RD$ for some orthogonal and diagonal matrices $R$ and $D$ respectively.
In practice, one does not have exact knowledge of the cumulant tensor, and
the cumulant tensor thus needs to be estimated from samples.  
Cumulants can
be estimated using $k$-statistics, which are unbiased estimates of
cumulants.  $k$-statistics have been studied within the statistics
community, and are discussed in chapter 4 of
\cite{McCullagh1987}.  For the fourth order cumulant tensor, given
random variables $\vec{Y}_1, \vec{Y}_2, \dotsc, \vec{Y}_n$, the
$k$-statistic $k(\vec{Y}_i, \vec{Y}_j, \vec{Y}_k, \vec{Y}_l)$, which
estimates $\cum(\vec{Y}_i, \vec{Y}_j, \vec{Y}_k, \vec{Y}_l)$, is:
  \begin{equation*} 
   k(\vec{Y}_i, \vec{Y}_j, \vec{Y}_k, \vec{Y}_l) = \frac{1}{N} \sum_{r,s,t,u=1}^N \phi(r,s,t,u) \vec{y}_i^{(r)}\vec{y}_j^{(s)}\vec{y}_k^{(t)}\vec{y}_l^{(u)},
\end{equation*}
  where $\phi$ is a function  invariant under permutations of its
  indices defined by $\phi(i, i, i, i) = 1$, $\phi(i, i, i, j) =
  \phi(i, i, j, j) = -1/(N-1)$, $\phi(i, i, j, k) = 2/[(N-1)(N-2)]$,
  and $\phi(i, j, k, l) = -6/[(N-1)(N-2)(N-3)]$ when $i, j, k, l \in [N]$
  are distinct \cite{McCullagh1987}.

  $k$-statistics share several
  important properties with the cumulant tensors that they estimate.
  The $k$-statistic is symmetric in that $k(X_i, X_j, X_k, X_l)$ is
  invariant under reordering of indices, and it is also multilinear.
  Multilinearity is shown for the fourth $k$-statistic in the
  following Lemma.
  \begin{lemma}\label{lem:kstat}
   The $k$-statistic transforms multilinearly.
  \end{lemma}
   \begin{proof}
     There are 2 properties of multilinearity.  For simplicity of
     notation, they will be only shown on the first coordinate of the
     k-statistic function.  Let $\vec{Y}_i, \vec{Y}_j, \vec{Y}_k,
     \vec{Y}_l, \vec{Z}_i$ be random variables, and let $c \in \R$.
     Then
     \begin{enumerate}
     \item The additivity portion of multilinearity comes from:
       \begin{align*}
         & k(\vec{Y}_i+\vec{Z}_i, \vec{Y}_j, \vec{Y}_k, \vec{Y}_l) \\
         &= \frac{1}{N}\sum_{r,s,t,u=1}^N \phi(r,s,t,u)
            (\vec{y}_i^{(r)}+\vec{z}_i^{(r)})\vec{y}_j^{(s)}\vec{y}_k^{(t)}\vec{y}_l^{(u)} \\
         &= \frac{1}{N}\left[\sum_{r,s,t,u=1}^N\phi(r,s,t,u)
            \vec{y}_i^{(r)}\vec{y}_j^{(s)}\vec{y}_k^{(t)}\vec{y}_l^{(u)}
            \jrnvers{ \right. \notag   \\
           & \quad \quad \quad  \left.  } +
            \sum_{r,s,t,u=1}^N \phi(r,s,t,u)\vec{z}_i^{(r)}\vec{y}_j^{(s)}\vec{y}_k^{(t)}\vec{y}_l^{(u)} \right] \\
         &= k(\vec{Y}_i, \vec{Y}_j, \vec{Y}_k, \vec{Y}_l)
            + k(\vec{Z}_i, \vec{Y}_j, \vec{Y}_k, \vec{Y}_l).
       \end{align*}
     \item  The multiplicative portion of multilinearity comes from:
       \begin{align*}
\jrnvers{         & k(c\vec{Y}_i, \vec{Y}_j, \vec{Y}_k, \vec{Y}_l) \\ }
\stdvers{k(c\vec{Y}_i, \vec{Y}_j, \vec{Y}_k, \vec{Y}_l)}
         &= \frac{1}{N}\sum_{r,s,t,u=1}^N \phi(r,s,t,u)c\vec{y}_i^{(r)}\vec{y}_j^{(s)}\vec{y}_k^{(t)}\vec{y}_l^{(u)} \\
         &= c\frac{1}{N}\sum_{r,s,t,u=1}^N \phi(r,s,t,u)\vec{y}_i^{(r)}\vec{y}_j^{(s)}\vec{y}_k^{(t)}\vec{y}_l^{(u)} \\
         &= k(\vec{Y}_i, \vec{Y}_j, \vec{Y}_k, \vec{Y}_l).
       \end{align*}
     \end{enumerate}
   \end{proof}

  These multilinearity properties imply that
  \begin{align*}
    & k(\vec{X}_i, \vec{X}_j, \vec{X}_k, \vec{X}_l) \\
      &= \sum_{qrst}k(A_{iq}(\vec{S} + A^{-1}\gvec{\eta})_q, A_{jr}(\vec{S} + A^{-1}\gvec{\eta})_r, \jrnvers{ \\
      &\quad \quad\quad\quad } A_{ks}(\vec{S} + A^{-1}\gvec{\eta})_s, A_{lt}(\vec{S} + A^{-1}\gvec{\eta})_t) \\
      &= \sum_{qrst}A_{iq}A_{jr}A_{ks}A_{lt}k(\vec{S}_q + (A^{-1}\gvec{\eta})_q, \jrnvers{ \\ 
      &\quad \quad\quad\quad }\vec{S}_r + (A^{-1}\gvec{\eta})_r, \vec{S}_s + (A^{-1}\gvec{\eta})_s, \vec{S}_t + (A^{-1}\gvec{\eta})_t).
  \end{align*}
  As such, Lemma \ref{lemma:COV_Tranformation} applies also to $k$-statistic estimates
  of random variables.  In particular, it is possible to think of the $k$-statistic tensor
  associated with the random variable $\vec{\vec{X}}$ as being
  generated from an unobserved $k$-statistic tensor from the latent samples of
  $\vec{S} + A^{-1}\gvec{\eta}$.   We can work directly with the random
  variable $\vec{S} + A^{-1}\gvec{\eta}$ for the purposes of error analysis.  This will be a natural approach
  since the difficulty of the problem relies partially on the \replacement{JV}{kurtosis}{fourth cumulant} of the latent
  distribution for $\vec{S} + A^{-1}\gvec{\eta}$.


  Let $\mu_k$ represent $\max_i\E[\vec{S}_i^k]$.  By assumption,
  $\mu_1 = 0$ and $\mu_2 = 1$.  Let $\gvec{\eta}^*$ denote
  $A^{-1} \gvec{\eta}$.  Let \[\sigma_{\gvec{\eta}^*} =
  \max_{\norm{u}=1}\left(\sqrt{\vec{u}^T\Sigma_{\gvec{\eta}^*}\vec{u}}\right)\]
  where $\Sigma_{\gvec{\eta}^*}$ is the covariance matrix of
  $\gvec{\eta}^*$.  The error induced by estimating the latent
  fourth cumulant tensor $k_{\vec{S}+\gvec{\eta}^*}$ from
  a sample can be bounded using the following 2 Lemmas:

  \begin{lemma} \label{lemma:k-stat-variance-bound}
    Let $\vec{Z} = \vec{S} + \gvec{\eta}^*$.  Then, \[\var(k(\vec{Z}_i, \vec{Z}_j,
    \vec{Z}_k, \vec{Z}_l)) = O\left(\frac{\max_{i\in [n]}\E[Z_i^8]}{N}\right).\]
  \end{lemma}
    \begin{proof}
      In order to save space, it will be useful to use multi-index notation.
      In particular, taking $I = (i_1, i_2, i_3, i_4) \in [n]^4$ and $\alpha
      = (\alpha_1, \alpha_2, \alpha_3, \alpha_4) \in [N]^4$,
      $\phi_{\alpha}\vec{z}_I^{(\alpha)}$ will be denote
      $$\phi(\alpha_1, \alpha_2, \alpha_3, \alpha_4) \vec{z}_{i_1}^{(\alpha_1)}
      \vec{z}_{i_2}^{(\alpha_2)}\vec{z}_{i_3}^{(\alpha_3)}\vec{z}_{i_4}^{(\alpha_4)}$$
      Further, the set $\alpha \cap \beta$ will be defined as:
    \begin{equation*}
       \alpha \cap \beta = \{ \alpha_i : \alpha_i = \beta_j \text{ for some pair }
       (i, j) \}.
    \end{equation*}
      Keeping these notations in mind, we can proceed with the proof.  Let
      $I\in [n]^4$.
      \begin{align}
       &\var(k(\vec{Z}_{i_1}, \vec{Z}_{i_2}, \vec{Z}_{i_3}, \vec{Z}_{i_4})) \nonumber \\
       &= \E\left[ \left(\frac{1}{N}\sum_{\alpha \in [N]^4} \phi_{\alpha} \vec{z}_I^{(\alpha)}\right)^2 \right]
          - \E\left[\frac{1}{N}\sum_{\alpha \in [N]^4} \phi_{\alpha} \vec{z}_I^{(\alpha)} \right]^2 \nonumber \\
       &= \frac{1}{N^2}\sum_{\alpha \in [N]^4}\sum_{\beta \in [N]^4}
         \E[\phi_{\alpha} \vec{z}_I^{(\alpha)} \phi_{\beta} \vec{z}_I^{(\beta)} ] \nonumber \jrnvers{ \\
         & \quad \  }
         - \frac{1}{N^2} \sum_{\alpha \in [N]^4}\sum_{\beta \in [N]^4}
         \E[\phi_{\alpha} \vec{z}_I^{(\alpha)}] \E[\phi_{\beta} \vec{z}_I^{(\beta)} ] \nonumber \\
       &= \frac{1}{N^2}\sum_{\alpha \in [N]^4}
          \sum_{\substack{\beta \in [N]^4 \\ \alpha \cap \beta \neq \emptyset}}
         \E[\phi_{\alpha} \vec{z}_I^{(\alpha)} \phi_{\beta} \vec{z}_I^{(\beta)} ] \nonumber \jrnvers{ \\
         & \quad \  }
         - \frac{1}{N^2} \sum_{\alpha \in [N]^4}\sum_{\substack{\beta \in [N]^4 \\ \alpha \cap \beta \neq \emptyset}}
         \E[\phi_{\alpha} \vec{z}_I^{(\alpha)}] \E[\phi_{\beta} \vec{z}_I^{(\beta)} ] \nonumber \\
       &\leq \frac{1}{N^2}\sum_{\alpha \in [N]^4} \phi_{\alpha}
         \sum_{\substack{\beta \in [N]^4 \\ \alpha \cap \beta \neq \emptyset}} \phi_{\beta}
         \E[ \vec{z}_I^{(\alpha)}  \vec{z}_I^{(\beta)} ]. \label{eq:k-stat-variance-sum}
      \end{align}
      Equation \eqref{eq:k-stat-variance-sum} contains the essence of the argument.  However, in
      order to complete the argument, several facts need to be demonstrated.  First, it needs to
      be seen that $\abs{\E[\vec{z}_I^{(\alpha)}\vec{z}_I^{(\beta)}]} \leq \max_i(\E[\vec{Z}_i^8])$.
      To see this, use the Cauchy-Schwartz inequality on random variables $Y_1, Y_2$ to get:
      \begin{equation} \label{eq:expectationInequality}
        \E[Y_1 Y_2] \leq \max (\E[Y_1^2], \E[Y_2^2])
      \end{equation}
      Applying this fact recursively yields that
      $\abs{\E[\vec{z}_I^{(\alpha)}  \vec{z}_I^{(\beta)}]} \leq \max_i(\E[\vec{Z}_i^8])$.

      The second difficulty that arises is seeing how limiting oneself to samples in which
      $\alpha \cap \beta \neq \emptyset$ restricts the summation.  First, let $\dist(\beta)$ denote
      the number of distinct indices in $\beta$.  If $c = \dist(\beta)$, then there are $\nCr{N}{c}$
      choices of index values that can be used to generate $\beta$, of which $\nCr{N-4}{c}$ certainly
      do not intersect $\alpha$.  As such,
      \begin{equation*}
        \frac{\nCr{N}{c} - \nCr{N-4}{c}}{\nCr{N}{c}}
      \end{equation*}
      gives an upper bound on the fraction of index sets in which $\beta \cap \alpha \neq \emptyset$
      when $\dist(\beta) = c$.  Finally, noting that
      $\abs{\phi_{\beta}} \leq 7/(N^{\dist(\beta)-1})$ for sufficiently large $N$ and that
      $\sum_{\alpha\in[N]^4}\phi_{\alpha} = O(N)$, we have sufficient tools with which to proceed from \eqref{eq:k-stat-variance-sum}:
      \begin{align*}
\jrnvers{       &\var(k(\vec{Z}_{i_1}, \vec{Z}_{i_2}, \vec{Z}_{i_3}, \vec{Z}_{i_4})) \\ }
\stdvers{       \var(k(\vec{Z}_{i_1}, \vec{Z}_{i_2}, \vec{Z}_{i_3}, \vec{Z}_{i_4})) }
       &\leq \frac{1}{N^2}\abs{\sum_{\alpha \in [N]^4} \phi_{\alpha} \sum_{c = 1}^4
         \sum_{\substack{\dist(\beta) = c \\ \alpha \cap \beta \neq \emptyset}} \phi_{\beta}
         \E[ \vec{z}_I^{(\alpha)}  \vec{z}_I^{(\beta)} ]} \\
        	& \leq \frac{1}{N^2}\max_{i\in [n]}\E[\vec{Z}_i^8]
         \sum_{\alpha \in [N]^4} \abs{\phi_{\alpha}} \sum_{c = 1}^4
         \sum_{\substack{\dist(\beta) = c \\ \alpha \cap \beta \neq \emptyset}} \abs{\phi_{\beta}} \\
       &\leq \frac{1}{N^2}\max_{i \in [n]} \E[\vec{Z}_i^8]
         \sum_{\alpha \in [N]^4} \abs{\phi_{\alpha}} \sum_{c = 1}^4
         \frac{\nCr{N}{c} - \nCr{N-4}{c}}{\nCr{N}{c}} 7N^{-c+1} \sum_{\dist(\beta) = c} 1\\
       &= O\left( \frac{1}{N^2}\max_{i \in [n]} (\E[\vec{Z}_i^8]) N^{-1} N^{-c+1} N^c
         \sum_{\alpha \in [N]^4} \abs{\phi_{\alpha}} \right)  \\
       &= O\left(\frac{\max_{i \in [n]} (\E[\vec{Z}_i^8])}{N}\right).
      \end{align*}
    \end{proof}

  \begin{lemma} \label{lemma:k_statistic_Sampling_Bound}
    Given $\epsilon, \delta > 0$, the error of each term in the
    $k$-statistic tensor for $\vec S + A^{-1}\gvec{\eta}$ is at most
    $\epsilon$ with probability $1 - \delta$ using
    \begin{equation*}
      N = O\left( \frac{n^4}{\epsilon^2 \delta} \biggl(\mu_{8} + \frac{\sigma^8_{\gvec{\eta}}}{\sigma_{\min}(A)^8}\biggr)
               \right)
    \end{equation*}
    samples.
  \end{lemma}

   \begin{proof}
     Define $\vec{Z} = \vec{S} + \gvec{\eta}^*$. Then using Lemma
     \ref{lemma:k-stat-variance-bound}\stdvers{}, $\var(k(\vec{Z}_i, \vec{Z}_j, \vec{Z}_k, \vec{Z}_l)) = O(\frac{1}{N}\max_{q \in [n]}\E[\vec{Z}_q^8])$.
     Using the binomial expansion,
      \begin{align*}
     \E[\vec{Z}_q^8]
       &= \sum_{m = 0}^8 \nCr{8}{m} \E[\vec{S}_q^m(\gvec{\eta}^*_q)^{8-m}] \\
       &= \sum_{m = 0}^4 \nCr{8}{2m} \E[\vec{S}_q^{2m}(\gvec{\eta}^*_q)^{8-2m}],
      \end{align*}
    since odd 0-mean Gaussian moments are 0.  Using equation \eqref{eq:expectationInequality}, we see that the dominant terms are
    $\mu_8$ and $(\sigma_{\eta}/\sigma_{\min}(A))^8$.  In particular, the cross terms come from
    $\E[\vec{S}_q^{2m}\gvec{\eta}^*_{8-2m}]$.  When $m = 2$, from \eqref{eq:expectationInequality},
    it follows that 
    \[\E[\vec{S}_q^4(\gvec{\eta}^*_q)^{4}] \leq \max( \mu_8, \E[(\gvec{\eta}^*_q)^{8}] ).\]
    When $m=1$, then 
    \begin{align*}
    \E[\vec{S}_q^2(\gvec{\eta}^*_q)^{6}] 
    &= \E[(\vec{S}_q^2(\gvec{\eta}^*_q)^{2})(\gvec{\eta}^*_q)^{4}] \\
    &\leq \max(\E[\vec{S}_q^4(\gvec{\eta}^*_q)^{4})], \E[(\gvec{\eta}^*_q)^{8})],     
    \end{align*}
for which
    $\E[\vec{S}_q^4(\gvec{\eta}^*_q)^{4})]\leq \max( \mu_8, \E[(\gvec{\eta}^*_q)^{8}] )$ has just
    been shown.  The case $m=3$ can be argued similarly to $m=1$ interchanging the roles of
    $\vec{S}$ and $\gvec{\eta}^*$.  Thus, one gets:
    \begin{equation*}
     \E[\vec{Z}_q^8]
        = O(\mu_8 + \E[(\gvec{\eta}^*_q)^{8}]). 
    \end{equation*}
    For even Gaussian moments, the following equation holds 
    (see for instance \cite{Kendall94} section 3.4):
    \begin{equation*}
      \E[\sigma_{\gvec{\eta}^*}^{2k}] = \frac{(2k)!}{k!2^k} \sigma_{\gvec{\eta}^*}^{2k}.
    \end{equation*}
    It follows that
    \begin{align*}
     \E[\vec{Z}_q^8]
        &= O(\mu_8 + \sigma_{\gvec{\eta}^*}^8) \\
        &= O(\mu_8 + \sigma_{\gvec{\eta}}^8/\sigma_{\min}(A)^8).
    \end{align*}
    Chebyshev's inequality states that for a random variable $Y$,
    $\Pr(\abs{Y - \mu_Y} \geq c\sigma_Y) \leq \frac{1}{c^2}$.  Taking $Y$ to be
    $k(\vec{S}_i + \gvec{\eta}^*_i, \vec{S}_j + \gvec{\eta}^*_j,
     \vec{S}_k + \gvec{\eta}^*_j, \vec{S}_l + \gvec{\eta}^*_l)$, then since
    the $k$-statistic is unbiased, it follows that its expectation is
    $\cum(\vec{S}_i + \gvec{\eta}^*_i, \vec{S}_j + \gvec{\eta}^*_j,
     \vec{S}_k + \gvec{\eta}^*_j, \vec{S}_l + \gvec{\eta}^*_l) =
    \cum(\vec{S}_i, \vec{S}_j, \vec{S}_k, \vec{S}_l)$.  $c$ can be chosen such
    that $\delta/n^4 \geq 1/c^2$.  Then, in order to bound the error beneath $\epsilon$,
    it suffices to satisfy:
    \begin{equation*}
       \epsilon \geq c \sqrt{\var(k(\vec{S}_i + \gvec{\eta}^*_i, \vec{S}_j + \gvec{\eta}^*_j,
     \vec{S}_k + \gvec{\eta}^*_j, \vec{S}_l + \gvec{\eta}^*_l))},
    \end{equation*}
    which can be guaranteed by choosing $N$ such that
    $\epsilon \geq c O((\frac{1}{N}\max_{q\in [n]}(\E[\vec{Z}_i^8]))^{1/2})$.
    This leads to the expression:
    \begin{align*}
      c O\left(\sqrt{\frac{1}{N}\max_{q\in [n]}(\E[\vec{Z}_i^8])}\right)
              &\leq \epsilon \\
      O\left(\sqrt{\frac{\mu_8 + (\sigma^8_{\gvec{\eta}}/\sigma_{\min}(A)^8)}{N}
       }\right)\sqrt{\frac{n^4}{\delta}}
       & \leq \epsilon
    \end{align*}
    \begin{equation*}
      N \geq O\left(n^4 \frac{\mu_8 + (\sigma^8_{\gvec{\eta}}/\sigma_{\min}(A)^8)}{\epsilon^2 \delta}
       \right).
    \end{equation*}
    Applying the union bound, this number of samples is sufficient to guarantee with probability
    $1 - \delta$ that all terms in the $k$-statistic tensor for $\vec{S} + A^{-1}\gvec{\eta}$ can
    be bounded beneath $\epsilon$.
   \end{proof}

\section{Error Propagation for Quasi-\jrnvers{\\}Whitening}

\smallskip
What follows is an analysis for how error propagates throughout the
quasi-whitening algorithm.  It will be demonstrated that
the canonical vectors which act as a basis for the independent components of
$\vec{S}$ will remain approximately orthogonal after
quasi-whitening given sufficiently many samples.  It will be
demonstrated that the required number of samples is polynomial in
terms of $1/\epsilon$, $1/\delta$, $n$,
$\kappa(A)$, $\kmax/\kmin$, $1/\kmin$,
$\sigma_{\gvec{\eta}}/\sigma_{\min}(A)$, and $\mu_8$ where $\epsilon$ is the allowable
cosine error from orthogonality of the basis vectors, and $1 - \delta$ is the
probability of success.  Since it was
demonstrated in the previous section that, given any $\epsilon > 0$, the sample estimate of the
cumulant tensor can have error bounded by $\epsilon$ in each term, it
suffices to demonstrate that at each step of the algorithm, error does
not grow too fast.  The probability of success is unchanged since
only one sample is taken.  As a notation, hatted variables shall be
used to denote approximations of non-hatted variables.  It is assumed that
the $k$-statistic tensor $\hat{Q}_{\vec{S}}$ estimate of ${Q}_{\vec{S}}$ is defined from samples of the noisy latent variable
$\vec{S} + \gvec{\eta}^* = A^{-1}\vec{X}$, though for simplicity,  $\gvec{\eta}^*$ is suppressed from the subscript notation. (See also the discussion after Lemma \ref{lem:kstat}.)  
\insertion{JV}{Similarly, $\hat Q_{\vec X}$ comes from the $k$-statistic $k_{\vec X + \gvec \eta}$.}  \insertion{JV}{$\norm{\cdot}_F$ will denote the Frobenius norm. $\norm{\cdot}_{\max}$ will denote the max norm, i.e.:}
\begin{equation*}
    \norm{Q}_{\max} := \max_{i, j, k, l} \abs{Q_{ijkl}}
\end{equation*}
\insertion{JV}{for a fourth order tensor $Q$.}

\begin{lemma} \label{lemma:diagMatrixError}
  Given a sample of $\vec X$, let $\hat{Q}_{\vec{X}}$ and $\hat{Q}_{\vec{S}}$ be the associated $k$-statistic estimates for $Q_{\vec{X}}$ and $Q_{\vec{S}}$ respectively,
  and let $\hat{M}$ be an estimate for the matrix $M$ such that for some $\epsilon_1, \epsilon_2 > 0$,
  $\norm{\hat{Q}_{\vec{S}} - Q_{\vec{S}}}_{\max} \leq \epsilon_1$ and $\norm{\hat{M} - M}_2 \leq \epsilon_2$.
  There exists a matrix $Y$ such that $\hat{Q}_{\vec{X}} \circ \hat{M} = AYA^T$ and $Q_{\vec{X}} \circ M = ADA^T$ where $D$ is
  the diagonal matrix defined in Lemma \ref{lemma:cumulantDiagonalization}, 
  and the error in the estimate $Y$ is bounded as:
  \begin{align*}
    \norm{Y - D}_2 &\leq \norm{Y - D}_F  \\ 
    &\leq n^2\norm{A}_2^2 \norm{M}_F \epsilon_1 + \sqrt{n} \epsilon_2 \norm{A}_2^2(n^2\epsilon_1 + \kmax)
  \end{align*}
\end{lemma}
  \begin{proof}
     Using Lemma \ref{lemma:COV_Tranformation}, one gets
     $\hat{Q}_{\vec{X}} \circ \hat{M} = A (\hat{Q}_{\vec{S}} \circ (A^T \hat{M} A)) A^T$,
     which gives that $Y$ is well defined, and $Y = (\hat{Q}_{\vec{S}} \circ (A^T \hat{M} A))$.
     By similar reasoning, $D = Q_{\vec{S}} \circ (A^T M A)$.
     In the following investigation of error propagation,
     the tensors $\hat{Q}_\vec{S}$ and $Q_{\vec{S}}$ will be treated as matrices as described in section
     \ref{section:Quasi_Whitening}, and the 2-norm used on the tensors should be interpreted as
     if the tensor has been flattened to its $n^2 \times n^2$ matrix form.  Then:
     \begin{align*}
\jrnvers{       & \norm{Y - D}_F \\ }
\stdvers{\norm{Y - D}_F}
         &= \norm{ \hat{Q}_{\vec{S}} \circ (A^T \hat{M} A) - Q_{\vec{S}} \circ (A^T M A) }_F \\
         &= \norm{ \hat{Q}_{\vec{S}} \circ (A^T \hat{M} A) - Q_{\vec{S}} \circ (A^T \hat{M} A) \jrnvers{ \\
         & \quad \ \ } + Q_{\vec{S}} \circ (A^T \hat{M} A) - Q_{\vec{S}} \circ (A^T M A) }_F \\
         &\leq \norm{ \hat{Q}_{\vec{S}} - Q_{\vec{S}}}_2 \norm{(A^T \hat{M} A)}_F
                 + \norm{Q_{\vec{S}}}_2 \norm{A^T (\hat{M}-M) A}_F \\
         &\leq n^2\epsilon_1 \norm{A}_2^2\norm{\hat{M} - M + M}_F + \kmax \norm{A}_2^2 \sqrt{n} \epsilon_2 \\
         &\leq n^2 \norm{A}_2^2\epsilon_1(\norm{\hat{M} - M}_F + \norm{M}_F) + \sqrt{n} \kmax \norm{A}_2^2 \epsilon_2 \\
         &\leq n^2 \norm{A}_2^2\epsilon_1(\sqrt{n} \epsilon_2 + \norm{M}_F) + \sqrt{n}\kmax\norm{A}_2^2 \epsilon_2 \\
         & = n^2\norm{A}_2^2 \norm{M}_F \epsilon_1 + \sqrt{n} \epsilon_2 \norm{A}_2^2(n^2\epsilon_1 + \kmax).
     \end{align*}
     This is also a bound for $\norm{Y-D}_2$ based on the standard inequality $\norm{Y-D}_2 \leq \norm{Y-D}_F$.
  \end{proof}

Lemma \ref{lemma:diagMatrixError} above bounds the error growth from tensor operations
while placing all error on the diagonal matrix.
The next goal is to demonstrate that taking the inverse of a matrix has
reasonable error propagation properties.  The following Lemma (a portion of Theorem 2.5 from \cite{Stewart1990}) will be useful:
\begin{lemma} \label{lemma:Matrix_Inversion_Error_General}
  Let $\norm{\cdot}$ be any consistent matrix norm.  Given a matrix $C$ and a matrix perturbation $E$ such that $\norm{C^{-1} E} < 1$, and given $\tilde{C} = C+E$, then
  \begin{equation*}
    \frac{\norm{\tilde{C}^{-1} - C^{-1}}}{\norm{C^{-1}}}
     \leq \frac{\norm{C^{-1}E}}{1 - \norm{C^{-1}E}} \ .
  \end{equation*}
\end{lemma}
From this Lemma, it follows immediately that if $\norm{E}_2 \leq 1/(2\norm{C^{-1}}_2)$, then
\begin{equation} \label{eq:Matrix_Inversion_Error_General}
	\norm{\tilde{C}^{-1} - C^{-1}}_2 \leq 2 \norm{C^{-1}}^2_2 \norm{E}_2 \ .
\end{equation}



The main result of this paper is contained in Theorem \ref{thm:mainthm}, which we prove now.  
  \begin{proof}[{\proofword} of Theorem \ref{thm:mainthm}]
    The proof is split into 3 parts.  In the first part, the preceding
    Lemmas are used to propagate error from the estimated latent tensor $Q_{\vec{S}}$.
    Then, a bound on the number of samples required to bound within $\epsilon$ the cosine and 
    scaling errors for the basis for the independent subspace from equations \eqref{eq:mainthm:cosine} and \eqref{eq:mainthm:scaling}
    is stated.  Finally, it is demonstrated that the bound on angular error is correct.

    Let $N$ be a sample size to be chosen later as a function of an arbitrary parameter $\eta>0$, so that with probability $1-\delta$ we have $\norm{\hat{Q}_{\vec{S}} - Q_{\vec{S}}}_{\max} < \eta$.
    Then, let
    $D' = \diag(\kappa_4(\vec{S}_1) \norm{A_1}^2,\allowbreak  \dotsc,\allowbreak \kappa_4(\vec{S}_n) \norm{A_n}^2)$
    be the same as in the proof
    of Theorem \ref{thm:TensorQuasi-Whitening}.  By Lemma \ref{lemma:cumulantDiagonalization}, $AD'A^T = Q_{\vec{X}} \circ I$.
    Let $Y'$ be the estimate of $D'$ generated as $AY'A^T = \hat{Q}_{\vec{X}} \circ I$.
    Then by Lemma \ref{lemma:diagMatrixError}, it follows that
    $\norm{Y' - D'}_2 < n^{5/2} \norm{A}_2^2 \eta$.  In order to apply equation
    \eqref{eq:Matrix_Inversion_Error_General}, it is useful to get error bounds for 
    $\norm{D'^{-1}}_2$.  It can be shown that:
	\begin{equation} \label{eq:normD'_inv_bounds}
	\frac{1}{\kmax \sigma_{\max}(A)^2} \leq \norm{D'^{-1}}_2 \leq \frac{1}{\kmin \sigma_{\min}(A)^2}.
	\end{equation}
    Then, it follows that using equation \eqref{eq:Matrix_Inversion_Error_General}:
    \begin{align}
       \norm{Y'^{-1} - D'^{-1}}_2 
       &\leq 2 \norm{D'^{-1}}_2^2 \norm{Y' - D'}_2 \nonumber \\
       &\leq \frac{2n^{5/2}\norm{A}_2^2\eta}{\kmin^2\sigma_{\min}(A)^4} \nonumber \\
       &= \frac{2n^{5/2}\kappa(A)^2\eta}{\kmin^2\sigma_{\min}(A)^2} 
       \label{eq:mainthm:firstInversionError}
    \end{align}
    with the restriction that $\eta$ must be chosen such that
    $\norm{Y' - D'} = n^{5/2} \norm{A}_2^2 \eta \leq 1/(2\norm{D'^{-1}}_2)$.  
    This can be ensured by requiring that $\eta \leq \kmin / (2n^{5/2} \kappa(A)^2)$.

    Now, let $Y$ and $D$ be defined such that $ADA^T = Q_{\vec{X}} \circ (AD'A^T)^{-1}$ and
    $AYA^T = \hat{Q}_{\vec{X}} \circ (AY'A^T)^{-1}$.  By Lemma \ref{lemma:diagMatrixError},
    \begin{align*}
\jrnvers{      & \norm{Y-D}_2  \notag \\ }
\stdvers{ \norm{Y-D}_2 }
      & \leq n^2 \norm{A}_2^2 \norm{(AD'A^T)^{-1}}_F \eta \notag \\
      &  \quad  + \sqrt{n} \norm{(AY'A^T)^{-1} - (AD'A^T)^{-1}}_2 \norm{A}_2^2(n^2\eta + \kmax) \\
      & \leq n^2 \kappa(A)^2\norm{D'^{-1}}_F\eta \\&\quad+ \sqrt{n}\kappa(A)^2\norm{Y'^{-1} - D'^{-1}}_2(n^2\eta + \kmax) \\
      & \leq  \frac{n^{5/2}\kappa(A)^2}{\kmin \sigma_{\min}(A)^2}\eta
              + \frac{2n^3\kappa(A)^4\kmax}{\kmin^2 \sigma_{\min}(A)^2}\eta 
              + \frac{2n^{5}\kappa(A)^4}{\kmin^2\sigma_{\min}(A)^2}\eta^2
    \end{align*} 
    which follows by applying \eqref{eq:normD'_inv_bounds} and \eqref{eq:mainthm:firstInversionError}.  
    
    Since $\eta \leq \kmin / (2n^{5/2} \kappa(A)^2)$,
    \begin{equation*}
      \norm{Y-D}_2 = O\left(\frac{n^3\kappa(A)^4\kmax}{ \sigma_{\min}(A)^2 \kmin^2}\eta\right).
    \end{equation*}

    Once again, it will be necessary to bound $\norm{D}_2$ in order to apply equation \eqref{eq:Matrix_Inversion_Error_General}.  Using $D = \diag(1/\norm{A_1}_2^2, \dotsc, 1/\norm{A_n}_2^2)$ from the proof of Theorem
    \ref{thm:TensorQuasi-Whitening}, it follows that:
    \begin{equation*}
      \sigma_{\min}(A)^2 \leq \norm{D^{-1}}_2 \leq \sigma_{\max}(A)^2.
    \end{equation*}

    Applying equation \eqref{eq:Matrix_Inversion_Error_General} yields:
    \begin{align*} 
      \norm{Y^{-1} - D^{-1}}_2 
      &\leq 2\norm{D^{-1}}_2^2 \norm{Y-D}_2 \\
      &\leq 2\sigma_{\max}(A)^4 O\left(\frac{n^3 \kappa(A)^4 \kmax}{\sigma_{\min}(A)^2 \kmin^2} \eta \right) \\
      \frac{\norm{Y^{-1} - D^{-1}}_2}{\sigma_{\min}(A)^2} 
      &\leq O\left(\frac{n^3 \kappa(A)^8 \kmax}{\kmin^2} \eta \right)
    \end{align*}
    with the restriction that $\norm{Y-D}_2 \leq 1/(2\norm{D^{-1}}_2)$.  Noting
    that $\norm{Y-D}_2 = O\left(\frac{n^3\kappa(A)^4\kmax }{\sigma_{\min}(A)^2\kmin^2}\eta\right)$ and $1/\norm{D^{-1}}_2 \geq 1/\sigma_{\max}(A)^2$, it suffices
    to restrict $\eta \leq O\left( \frac{\kmin^2}{n^3 \kappa(A)^6 \kmax}\right)$.
%

    Since $\eta$ is arbitrary (except for upper bound restrictions), 
    $\eta$ can be chosen such that
    $\frac{\norm{Y^{-1}-D^{-1}}_2}{\sigma_{\min}(A)^2} < \frac{\epsilon}{2}$.
    This can be accomplished taking $\eta = O\left(\frac{\kmin^2}{n^3 \kappa(A)^8 \kmax}\epsilon \right)$.  
    This choice is valid, as both restrictions on $\eta$ are met when $\epsilon \leq 1$.  By
    Lemma \ref{lemma:k_statistic_Sampling_Bound}, taking
    \begin{align*}
      N &= O\left(\frac{n^4}{\eta^2 \delta}(\mu_8 + (\sigma^8_{\gvec{\eta}}/\sigma_{\min}(A)^8))\right) \\
        &= O\left(n^{10} \frac{(\kappa(A)^{16}\kmax^2)}{\epsilon^2\delta \kmin^4}(\mu_8 + (\sigma^8_{\gvec{\eta}}/\sigma_{\min}(A)^8))\right)
    \end{align*}
    samples suffice to obtain the desired error bound $\epsilon$ with probability $1-\delta$.

    The basis in which $\vec S$ has independent coordinates is the canonical basis. Therefore, the ultimate goal is to show that, with our choice of an approximate quasi-whitening matrix $\hat B^{-1}$ below, the canonical vectors stay approximately orthogonal after applying $\hat B^{-1} A$.
    To see this, factorize
    $\hat{B}\hat{B}^T = \hat{Q}_{\vec{X}} \circ (\hat{Q}_{\vec{X}} \circ I)^{-1}$.
    $\hat{B}^{-1}$ is the approximate quasi-whitening matrix, and $\hat{B}\hat{B}^T = AYA^T$
    gives that $\hat{B}^{-1}AY^{1/2} = R$ for some orthogonal matrix $R$,
    and $\hat{B}^{-1}A = RY^{-1/2}$.  Since $Y$ is
    symmetric, $Y^{-1/2}$ can be taken to be a symmetric matrix.  Take
    $\vec{e}_i, \vec{e}_j$ to be canonical vectors.  Define $\delta_{ij}$ to be the delta function such that 
    \[ 
    \delta_{ij} =  \begin{cases}
            1 & \text{if $i = j$}, \\
            0 & \text{otherwise}.
    \end{cases}
    \]
      Then with probability $1 - \delta$,
    \begin{align*}
     \jrnvers{ & } \frac{\langle \hat{B}^{-1}A \vec{e}_i, \hat{B}^{-1}A \vec{e}_j\rangle}{\norm{A_i}_2\norm{A_j}_2} \jrnvers{ \\ }
      & =\frac{\vec{e}_i^T A^T \hat{B}^{-T}\hat{B}^{-1}A\vec{e}_j}{\norm{A_i}_2\norm{A_j}_2} \\
      & =\frac{\vec{e}_i^T Y^{-1} \vec{e}_j}{\norm{A_i}_2\norm{A_j}_2} \\
      & \in \left(\frac{D^{-1}_{ij}}{\norm{A_i}_2 \norm{A_j}_2} - \frac{\epsilon}{2} ,
                   \frac{D^{-1}_{ij}}{\norm{A_i}_2 \norm{A_j}_2} + \frac{\epsilon}{2} \right) \\
      & \supset \left(\frac{\delta_{ij} \norm{A_i}\norm{A_j}}{\norm{A_i}\norm{A_j}} - \frac{\epsilon}{2} ,
              \frac{\delta_{ij} \norm{A_i}\norm{A_j}}{\norm{A_i}\norm{A_j}} + \frac{\epsilon}{2} \right) \\
      & = \delta_{ij} \pm \frac{\epsilon}{2}.
    \end{align*}
    Consider the case where $i = j$.  Then,
    \begin{equation*}
    	\norms{\hat{B}^{-1}A\vec{e}_i}_2 \in 
    	\left(1 \pm \frac{\epsilon}{2}\right) \norms{A_i}_2,
    \end{equation*}
    which gives equation \eqref{eq:mainthm:scaling}
    Consider the case where $i \neq j$.  Then,
    \begin{equation*}
    \frac{\langle \hat{B}^{-1}A \vec{e}_i, \hat{B}^{-1}A \vec{e}_j\rangle}{\norm{A_i}_2\norm{A_j}_2} \cdot 
    \frac{\norm{\hat{B}^{-1}A \vec{e}_i}_2\norm{\hat{B}^{-1}A \vec{e}_j}_2}{\norm{\hat{B}^{-1}A \vec{e}_i}_2\norm{\hat{B}^{-1}A \vec{e}_j}_2}
    \in \pm \frac{\epsilon}{2} \\
    \end{equation*}
    \begin{align*}
    \frac{\langle \hat{B}^{-1}A \vec{e}_i, \hat{B}^{-1}A \vec{e}_j\rangle}
    {\norm{\hat{B}^{-1}A \vec{e}_i}_2\norm{\hat{B}^{-1}A \vec{e}_j}_2}
    &\in \pm \frac{\epsilon}{2} \cdot \frac{1}{1 \pm \epsilon/2} \\
    &\subset \pm \epsilon
    \end{align*}
    by restricting $\epsilon < \frac{1}{2}$.  This gives equation \eqref{eq:mainthm:cosine},
    completing the proof.
  \end{proof}

\section{Acknowledgments}
We would like to thank Navin Goyal for useful discussions.
\insertion{MB}{Mikhail Belkin and James Voss were partially supported
  by NSF Grants IIS 0643916 and IIS 1117707 during the writing of the
  paper.}

\bibliographystyle{abbrv}    

\balancecolumns
\bibliography{gaussianRobustICA}

\end{document}